\def\year{2023}\relax
\documentclass[letterpaper]{article} 
\usepackage{aaai23}  
\usepackage{times}  
\usepackage{helvet}  
\usepackage{courier}  
\usepackage[hyphens]{url}  
\usepackage{graphicx} 
\usepackage{natbib}  
\usepackage{caption} 
\usepackage{algorithm}

\usepackage{newfloat}
\usepackage{listings}
\DeclareCaptionStyle{ruled}{labelfont=normalfont,labelsep=colon,strut=off} 
\floatstyle{ruled}
\newfloat{listing}{tb}{lst}{}
\floatname{listing}{Listing}
\usepackage[utf8]{inputenc} 
\usepackage[T1]{fontenc}    
\usepackage{hyperref}       
\usepackage{url}            
\usepackage{booktabs}       
\usepackage{amsfonts}       
\usepackage{nicefrac}       
\usepackage{microtype}      

\usepackage{times}

\usepackage{multicol}
\usepackage{algorithm}
\usepackage{setspace}
\usepackage{algpseudocode}

\usepackage{amsmath}
\usepackage{amssymb}
\usepackage{graphicx}
\usepackage{mathtools}
\usepackage{tikz}
\usepackage{times}
\usepackage{subfigure}
\usepackage{multirow}
\usepackage{graphics}
\usepackage{booktabs}
\usetikzlibrary{calc,shapes}

\usepackage{amsthm}

\newtheorem{proposition}{Proposition}
\newtheorem{definition}{Definition}

\usepackage{caption}
\usepackage{graphicx}
\usepackage{booktabs}
\usepackage{array}

\usepackage{siunitx,etoolbox}

\renewrobustcmd{\bfseries}{\fontseries{b}\selectfont}
\renewrobustcmd{\boldmath}{}

\title{Connection-Based Scheduling for Real-Time Intersection Control}  

\author {
    Hsu-Chieh Hu,\textsuperscript{\rm 1}
    Joseph Zhou, \textsuperscript{\rm 1}
    Gregory J. Barlow, \textsuperscript{\rm 1}
    Stephen F. Smith, \textsuperscript{\rm 1,2}\\
}
\affiliations {
    \textsuperscript{\rm 1} Rapid Flow Technologies Inc., Pittsburgh, PA\\
    \textsuperscript{\rm 2} Carnegie Mellon University, Pittsburgh, PA\\
    \{hsuchieh, jqz, gjb\}@rapidflowtech.com, sfs@cs.cmu.edu\\
    \textbf{August 15, 2022}
}

\begin{document}

\maketitle
\begin{abstract}  
We introduce a heuristic scheduling algorithm for real-time adaptive traffic signal control to reduce traffic congestion. This algorithm adopts a lane-based model that estimates the arrival time of all vehicles approaching an intersection through different lanes, and then computes a schedule (i.e., a signal timing plan) that minimizes the cumulative delay incurred by all approaching vehicles. State space, pruning checks and an admissible heuristic for A* search are described and shown to be capable of generating an intersection schedule in real-time (i.e., every second). Due to the effectiveness of the heuristics, the proposed approach outperforms a less expressive Dynamic Programming approach and previous A*-based approaches in run-time performance, both in simulated test environments and actual field tests. 
\end{abstract}

\section{Introduction}
As the number of vehicles around the world continues to increase along with the population, traffic congestion is already a serious problem for most urban areas. It is generally recognized that better optimization of traffic signals is crucial to future urban mobility. With the advancement of sensing and vehicular technologies, it is time to rethink current approaches to traffic control systems, that were designed decades ago, and by and large still assume analog hardware components. One promising approach that has emerged in recent years is {\em schedule-driven traffic control}. \cite{Xie2012,xie2012schedule,smith2013smart}. Under this approach, traffic signal control is formulated as a decentralized online planning process. Each intersection utilizes real-time sensing to predict the future arrival times of approaching vehicles at the intersection and then solves a special type of single machine scheduling problem,   
where the intersection is the machine, input jobs are sequences of spatially proximate vehicle clusters representing queues and approaching platoons, and the objective is to minimize cumulative delay. Once an intersection computes its schedule, it communicates what traffic it expects to be sending to its neighbors to allow intersections to generate longer horizon plans over time and achieve coordination at the network level. The original work already demonstrated substantial improvements over traditional signal approaches in over $20$ North American cities and is continuously expanding its deployments to other areas.




One key factor to the effectiveness of the schedule-driven traffic control algorithm is an ability to accurately predict when sensed vehicles will arrive at and pass through the intersection \cite{hu2021incorporating}. Therefore, the use of a more expressive traffic flow model becomes an inevitable direction for improving the algorithm's ability to determine whether to extend or end a green signal adaptively in real-time. However, there is a trade-off between use of a higher fidelity model and real-time solvability. The exponentially larger number of traffic states that results from shifting from a single approach model to a lane-based model, for example, makes computation of an optimal schedule in real-time more challenging, especially, for complex intersections that have to serve multiple movements in parallel including left turns. Heuristic scheduling based on A* search is a promising way to solve this problem. In \cite{goldstein2019expressive}, a lane-based model was adopted, and a heuristic function was defined that exploits the structure of this new parallel machine scheduling problem formulation to enable generation of optimal schedules in real-time.

In this paper, we push the boundary of applying heuristic scheduling to traffic control further by improving the search model and proposing a more efficient heuristic function to speed up optimization. First, a new search model based on a lane-based traffic intersection formulation is presented. The new model enables the traffic state to be updated by individual vehicles within a vehicle cluster containing multiple geometrically adjacent vehicles, which means that a given input job is allowed to change its duration depending on the previous state during the search \cite{hu2021incorporating}, and scheduling decisions thus reflect more  realistic traffic dynamics. Second, a new heuristic function is proposed to solve the scheduling problem more efficiently. Instead of approximating the scheduling problem by multiple preemptive subproblems, as described in \cite{goldstein2019expressive}, we compute an exact solution to a relaxed version of the original lane-based scheduling problem with unit preemption \cite{morton1993heuristic}. Finally, four new pruning checks leveraging various timing constraints in the problem are introduced to speed up the search. The first two utilize maximum and minimum green time constraints on the length of signal phases respectively, the third exploits the reformulated dominance criterion for the new model, and the fourth draws on the cycle constraint used to determine next movements.  We demonstrate the potential of these through software profiling, simulation evaluations and field tests, showing the proposed model and heuristics outperform prior work from both efficiency (i.e., run-time and number of state expansions) and effectiveness (i.e., delay and number of stops) performance perspectives.

\section{Scheduling Problem Formulation}
We consider an intersection that has multiple lanes from different directions. A pair of inbound and outbound lanes that are connected by the intersection is defined as a \textit{connection}. This is the minimum unit for describing a movement (e.g., left turn of south bound) at the intersection. A subset of connections that allow vehicles to pass through the intersection concurrently is defined as a \textit{stage}. A dual ring barrier controller is an instance of the above model that groups two compatible movements with one stage, described in Figure \ref{stage}. A sequence of stages then constitutes a cycle, which indicates when the stage and the corresponding connections should turn to green. Each stage, denoted as $s$, has maximum and minimum green time limit constraints to ensure fairness, represented by $max_s$ and $min_s$. Similarly, fixed yellow and red clearances are inserted at each stage transition for maintaining safety. In the following discussion, we assume each movement (e.g., North-to-East) has only one connection to simplify the presentation. It is straightforward to extend the model to the case that one movement may contain multiple lanes (i.e., connections).

In the model, the incoming traffic flows are treated as sequences of clusters $c$ approaching the intersection over the planning horizon $H$. Vehicles, represented by $v$, that are taking the same connection to traverse the intersection are clustered to form $c$ if the separation time between them is less than a pre-specified time interval (e.g., $1$ second). To avoid oversized clusters, a cluster limit is also imposed (e.g., $10$ seconds). Each cluster $c$ is thus composed of multiple vehicles with related information, namely $v = (m, arr, dep, q) \in c$ of each vehicle, where $m$, $arr$, $dep$ and $q$ are the connection index, arrival time, departure time and a flag to determine if $v$ is in the queue respectively. The clusters then become the input jobs that must be sequenced through the intersection. Note that the cluster is the basic element for manageable sequencing, while the timing is still updated through each vehicle within the cluster $c$ in order to increase the granularity. The details will be explored in later sections. Once a vehicle departs the intersection, it is sensed and grouped into a new cluster by the downstream neighbor intersection. The sequences of clusters provide short-term variability of traffic flows for optimization and preserve the non-uniform nature of real-time flows.

More specifically, the input to the online planning process at the beginning of each planning cycle is a set of \textit{connection cluster sequences}, which is denoted as $C_m$ at the connection $m \leq M$, where $M$ represents the number of all connections at the intersection. During scheduling, each cluster is viewed as a non-divisible job that incorporates multiple vehicles and an A* search is executed in a rolling horizon fashion to continually generate a stage schedule that minimizes the cumulative delay of all clusters. In practice, the planning cycle is repeated every second, to reduce the uncertainty associated with clusters and queues. The process constructs an optimal sequence of clusters that maintains the ordering of clusters along each input lane $m$, and each time a stage change is implied by the sequence. It is worthwhile to note that the subset $m \in \mathcal{M}_s$ corresponding to the same stage $s$ can be served simultaneously, so the problem we are going to solve is a parallel machine scheduling problem. 

\begin{figure}[!htbp]
\centering
\includegraphics[scale = 0.18]{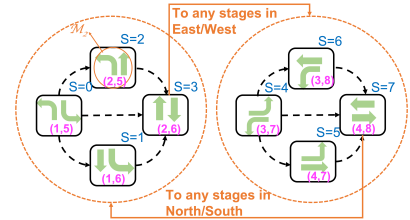}
\caption{Each stage has a subset of connections, and all stages constitutes a cycle through stage transitions. }
\label{stage}
\end{figure}

\section{Optimization Algorithm}
The scheduling module adopts A* search as the core optimization algorithm, given the sensed cluster sequences $C_m$ and current controller state. First, we define the search state and describe how to transition between states. Then, we introduce the heuristic to speed up search and reduce the size of the search space and explain how to calculate the heuristic function. Lastly, four pruning checks based on the timing constraints of traffic signals are described.     

\subsection{Search Overview}
The aim of the search is to find the optimal schedule (sequence) in terms of \textit{minimum cumulative delay} and compute each vehicle's departure time with respect to their arrival time without the interference of traffic signals and with sufficient time for crossing the intersection. To achieve this goal, the search algorithm requires the arrival time, the duration, and the queueing status of each vehicle. As mentioned above, the input ordered cluster sequences on each lane (connection) are represented as $C_m$. In addition, the controller state, including the current stage, the elapsed time, and the end time of each connection, is imported to create a root state. In the following sections, the search states and their corresponding transitions are defined, followed by a description of the root state and goal state.

\subsubsection{Search States}
During search, a schedule is generated by successively appending new clusters to the current partial schedule in time-forward order. Appending a new cluster $c$ will thus represents a  transition to a new search state. A search state is defined as tuple: $(s, m, sd, start_s, \mathbf{t}, \mathbf{q}, d, h)$, where $s$ is stage index, $m$ is connection index, $sd$ is duration of the most recent stage, $start_s$ is the start time of the stage $s$, $\mathbf{t} = (t_1,\cdots,t_{M})$ is the finish time on all connections. To estimate the queueing delay and increase the accuracy of the cumulative delay, tracking of queue count through online planning is applied \cite{hu2021incorporating}. To minimize the objective and support heuristic search, $d$ is the cumulative delay incurred so far, and $h$ is an underestimate of the remaining delay (i.e., heuristic). Other than the tuple, the number of served clusters on lane $m$, represented by $N_m$, are tracked and used to determine if the goal state in which all clusters are served is reached. The vehicle and cluster order is maintained for each $C_m$ since the vehicles cannot pass vehicles on the same lane.

As mentioned previously, each stage is composed of multiple connections (i.e., multiple input-output lane pairs). Since these connections are compatible traffic movements, they can be served simultaneously in parallel. We view $start_s$ as the earliest time for the vehicles on these connections to cross the intersection and apply it to update the end times $\mathbf{t}$. With $\mathbf{t}$, the actual start time of each vehicle within the cluster can be determined. If $t$ is greater than the $arr$ of the vehicle, the vehicle is queued, and the queue count is increased. After updating end times, stage duration $sd$ is the maximum of the end times on the lanes corresponding to the current stage.  

Each state also tracks a cumulative delay value $d$, which is the total delay incurred thus far. For A* search, $d$ is actually the $g$ value of the heuristic search. Also, a $h$ value which is lower bound of the remaining delay is tracked, and $d+h$ is equal to the $f$ value. The search state also maintains the connection index, which is used to determine the optimal schedule by tracing back the states. 

 \subsubsection{Root State}
 We start the search with a root state that is $(s_{curr}, m_{curr}, start_s, \mathbf{0}, \mathbf{q}_{curr}, 0, 0)$. The $s_{curr}$ and $m_{curr}$ are the current stage and connection, $\mathbf{q}_{curr}$ is the current queue counts of all lanes, and $\mathbf{t}$ and $h$ are initialized to $0$ to minimize delay incurred from now to the future.  

\subsubsection{Goal State}  
A goal state is one where all clusters (i.e., all vehicles within the $H$) have been serviced to pass through the intersection. More specifically, the number of serviced clusters on each lane is $N_m = |C_m|$, where $|C_m|$ is the number vehicles on the lane $m$.

\subsubsection{State Transitions}
When a new cluster $c$ is scheduled, a state transition will be incurred. The search state will be updated in Algorithm \ref{transition}. First, given a parent state $S_{p}$, a child state of serving $c$ is expanded. The vehicles $v$ within $c$ are then used to update the new state $S$ based upon the $S_{p}$.

 The algorithm first searches for the next cluster on each connection and uses them for expansions. To compute the child state given the parent state $S_p$ and a cluster $c$ on the connection $m$, Algorithm \ref{transition} determines if the stage is changed after serving $c$ by $nearestStage(m, s_p)$ (e.g.,  follow the transitions shown in Figure~\ref{stage}). If the stage change happens, we reset the stage duration $sd$ and stage start time $start_s$ for the child state. In addition, the end times will also have to be shifted by $minSwitch(s_p, s)$ which returns the minimum time required for switching from phase $s_p$ to $s$. 
 
Each cluster may contain multiple vehicles that are clustered based on the pre-defined time interval. In order to update the state in a detailed way, the state is updated by each vehicle within the cluster iteratively according to their $arr$. For instance, we estimate queueing delay and count accurately vehicle by vehicle to more accurately reflect a cluster's delay contribution. The cluster can also change its shape depending on the previous state during search. Queue count is similarly improved by serving individual vehicles within the cluster. On the other hand, we check if the cluster is over $max_s$ by comparing each vehicle's end time $t$ with $max_s$ and thus will not overestimate the end time. For each vehicle, their delay and end time are updated individually, while the computation is only multiplied by the max size of the clusters.
 
 After updating the child state, we calculate a lower bound of future delay as a heuristic estimate for A*'s $h$ value given the child state's end time $\mathbf{t}$, the remaining clusters $\mathbf{C}$ and the new count $\mathbf{N}$, where $\mathbf{C}$ and $\mathbf{N}$ are two arrays $(C_1, \cdots, C_M)$ and  $(N_1, \cdots, N_M)$. We then add the cumulative delay $d$, which is actually the $g$ value of A*, to $h$ to obtain the child state's $f$ value and add the child state to the sets of the explored states described in Algorithm \ref{expand}. Given the parent state, this update process repeats for each possible cluster on all possible connections. 
 
We present an example of a search state in Figure \ref{state} and its clusters on the connections in Figure \ref{timeline}.  For a partial schedule to the search state, the number of clusters being scheduled previously on connection $m$ is $N_m$, which is the $m_{th}$ element of $\mathbf{N} = [0, 2,1,1,1,2,2,1]$. We can only transition to connections that have remaining clusters (i.e., $N_m < |C_m|$), so only connection $1$, $2$, $6$, $3$ and $7$ are valid to be expanded from the state. If the stage sequence should obey the cycle according to Figure \ref{stage}, connection $1$ has to be served after stage $4$ that contains connection $3$ and $7$ and should be excluded for the transitions.

\begin{figure}
\begin{minipage}{.23\textwidth}
  \centering
  \includegraphics[scale=0.16]{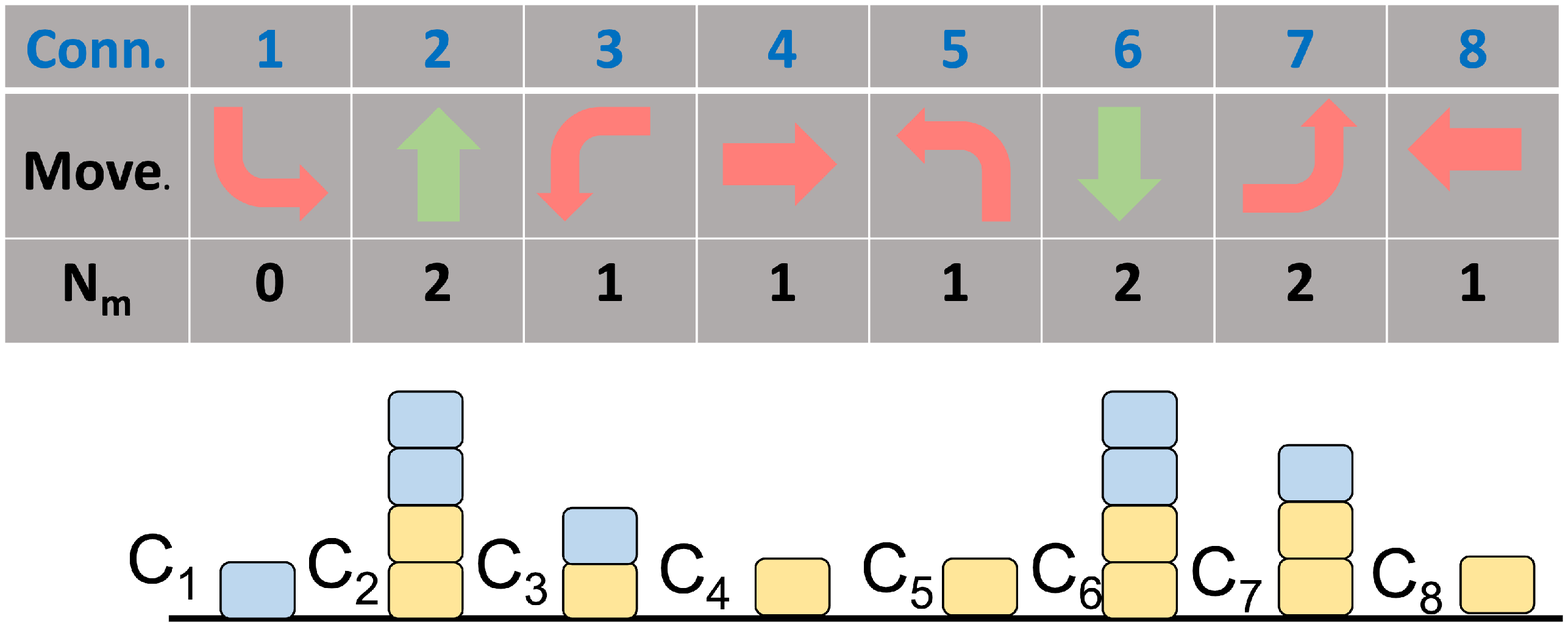}
  \captionof{figure}{A search state with $8$ connections and $8$ stages}
  \label{state}
\end{minipage}
\begin{minipage}{.23\textwidth}
  \centering
  \includegraphics[scale=0.18]{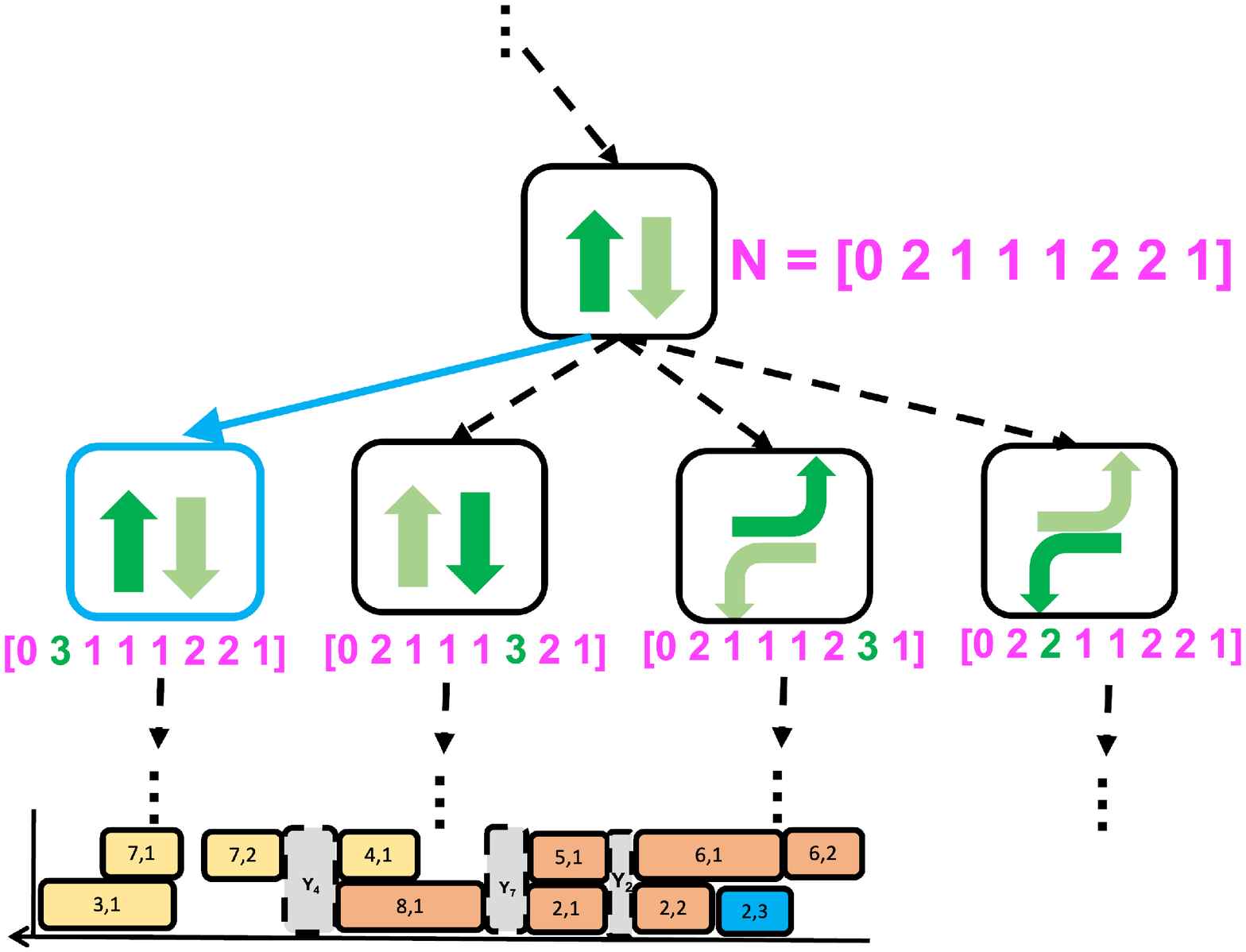}
  \captionof{figure}{State Transition Example}
  \label{tr}
\end{minipage}
\end{figure}


\begin{figure}[!htbp]
\centering
\includegraphics[scale = 0.22]{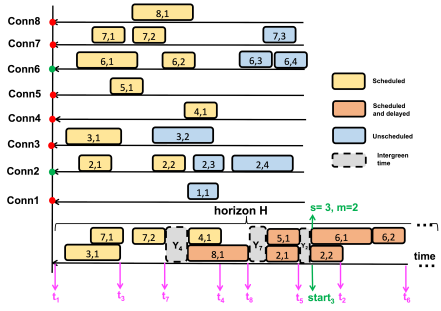}
\caption{Clusters on each connection and a partial schedule. The label on each cluster is $(connection, index)$. }
\label{timeline}
\end{figure}
\begin{algorithm}[tb]
\begin{algorithmic}[1]
\Require 1)$S_p = (s_p, m_p, sd_p, start_{s_p}, \mathbf{t}, \mathbf{q}, d, h)$ ; 2) $c\in C_m$ on connection $m$
\State $s = nearestStage(m, s_p)$, $sd = sd_p$
\State $t_m = t_m + minSwitch(s_p, s)$
\If{$s \neq s_{p}$}\quad$start_s = t_m, sd=0$\EndIf
\For{$v \in c$}
\If{$sd > max_s$}\quad \textbf{Break} \EndIf
\State$(dep(v), q_m)=queueingDelay(t_m, q_m)$
\State$d = d+ \max(t_m - arr(v), 0),N_m += 1$
\State $t_m = \max(arr(v), t_m) + dep(v) - arr(v)$
\State$sd = sd + t_m - start_s$
\EndFor
\State$h = calcHeuristic(s, \mathbf{t}, \mathbf{C}, \mathbf{N})$
\State\Return $S = (s, m, sd, start_{s}, \mathbf{t}, \mathbf{q}, d, h)$
\end{algorithmic}
\caption{$updateState(S_p, c)$}
\label{transition}
\end{algorithm}


The example in Figure \ref{timeline} represents the partial schedule (i.e. the search path) specifying clusters' arrival time and size on each connection. In this example, clusters contain either $1$ or $2$ vehicles. When we update the state with the cluster of size greater than $1$, each vehicle within the cluster will be used to update the state according to line $5$ of Algorithm \ref{transition}. From the partial schedule, we can know the corresponding end time on each connection and the future start time of each stage by delaying the clusters and adding the fixed duration of intergreen time (i.e., yellow and red clearance). For instance, stage $4$ is initially brought up and ending at $t_7$. Since there are no clusters on connection $1$, $t_1$ remains at $t_1 = 0$. Then, we switch to stage $7$ and know its start time and end time after adding the intergreen time $Y_4$ of stage $4$ and serving two clusters in which one is delayed (orange) and the other remains the free flow speed (yellow). It is the same to calculate the start time and end time of the following stage $2$ except both clusters are delayed. When the partial schedule ends with connection $2$ at stage $3$ (as shown in Figure \ref{timeline}), we can transition to stay at the same stage by expanding connection $2$,$6$ or immediately end the stage $3$ and switch to stage $4$ by serving connection $3$, $7$. As shown in Figure \ref{tr}, the new partial schedule is generated by expanding connection $2$ and staying in stage $3$. 
 
 \subsection{A Heuristic for Traffic Control}
 In this section, we introduce an algorithm to compute the $h$ value that is crucial to any A* search. The lower bound of future delay is obtained through solving a relaxation of the above scheduling problem. On the other hand, four pruning checks relevant to the constraints of traffic signal control are introduced to speed up the entire search process.
 \subsubsection{Heuristic Function}
To compute a lower bound on the future delay to be incurred, we define a heuristic function that solves a relaxation of the scheduling problem for the remaining decisions to be made (which corresponds to the set of remaining unscheduled clusters). The first step is to calculate the earliest start time of the first available unserved clusters on each available connection. Then, we convert this original problem into a corresponding dynamic scheduling problem with preemption, which historically has been used to achieve a good lower bound in branch-and-bound solutions to the nonpreemptive variant. Most importantly, the new relaxed problem has an exact solution \cite{morton1993heuristic}. In this section, a pipeline for computing the heuristic function is described.
 
First, we formulate the original dynamic scheduling problem with $C_m$ on each connection. We aim to search different combinations of clusters on the connections to minimize cumulative delay, which can be expressed as
\begin{equation}
\sum_{m = 1}^{M} \sum_{c\in C_m} \sum_{v\in c} d(v),
\end{equation}
where the delay $d(v)$ of each vehicle that contributes to the cumulative delay is $d(v) = \max(t_m - arr(v), 0)$.
This objective can be interpreted as weighted flow time in the formal dynamic single-machine model that minimizes the summation of the difference between completion time and ready time. This set of problems has exact results if preemption is allowed. To translate our problem to a dynamic scheduling problem, we take $arr(v)$ as the ready time and $t_m$ after serving $v$ as the completion time, and we aim to minimize the flow time, which is the amount of time vehicle $v$ spends in the system. The dynamic setting also reflects the fact that the vehicles arrive at different times to the local intersections. \cite{morton1993heuristic} described a priority rule to solve this problem with preemption, which is called preemptive dynamic weighted shortest processing time (PDWSPT), and we translate it to fit our problem: 
\begin{definition}
The PDWSPT priority rule for the traffic control is given as follows
\begin{enumerate}
\item When a vehicle is crossing the intersection, start the currently available vehicle with the highest $1/dur$, where $dur = dep - arr$. If we take the weights of the previous formed clusters \cite{hu2017icaps,hu2017softpressure} into account, the priority should be modified to $w/dur$ where $w$ is the pressure weights or the number of vehicles within the cluster.
\item When a vehicle $v_i$ arrives with another vehicle $v_j$ currently crossing the intersection, start $v_i$ provided $w_i/dur_i > w_j/dur_j$; otherwise continue serving $v_j$.
\label{pdwspt}
\end{enumerate}
\end{definition}
We can prove that the PDWSPT rule provides an exact optimal solution to the dynamic weighted flow problem.  
\begin{proposition}
The PDWSPT rule is optimal for the dynamic weighted flow problem.
\label{optimality}
\end{proposition}
\begin{proof}
We form an unit-preemptive setting in which each vehicle has weight $w/dur$ with a fixed $dur$ at an unit slot. Thus, $w/dur$ can be simplified to $w$ without loss of generality. According to the rule, the optimal sequence $v_1, \cdots, v_n$ must satisfy $w_i > w_{i+1}$. Suppose there exists an optimal solution that $w_i \leq w_{i+1}$. Interchange $v_i$ and $v_{i+1}$ produces a schedule with a delay decreased by $w_{i+1} - w_i$. Thus, the original schedule is not optimal. 
\end{proof}

As mentioned in the previous sections, we allow multiple connections to be served concurrently if the movements have no conflicts. A new rule is proposed to accommodate this assumption and described in Figure \ref{hfunction}. 1) We obtain the permitted start time for each connection similar to Algorithm \ref{transition} (line $1$ and $2$). 2) The vehicles of all unserved vehicles are shifted to the permitted start time, and the unavoidable delay is computed for all vehicles and added on the $h$ function. The complexity of PDWSPT implemented with priority queues is polynomial, so we choose to calculate the $h$ function by the individual vehicles as in Algorithm \ref{transition}. 3) Sort all vehicles on all connections according to their arrival times. 4) We propose a \textit{connection-based PDWSPT} priority rule for connection-based traffic control to compute a lower bound for the remaining problem: if two overlapped vehicle $v_m$ and $v_n$ on the connection $m$ and $n$ can be served without any conflicts (i.e., $n$ and $m$ are at the same stage), these two vehicles are then combined to form a cluster with a summed weight. The end time of the new formed cluster is updated by $\max(dep(v_m), dep(v_n))$. Otherwise, PDWSPT is applied. 
\begin{proposition}
The connection-based PDWSPT rule is optimal for the dynamic weighted flow problem.
\label{cpdwspt}
\end{proposition}
\begin{proof}
Suppose there exists an optimal solution where the two overlapped vehicles are not combined. Shifting the vehicle with the lower weight to the end of the other one will always increase the delay by itself and all later vehicles. Thus, solutions that don't combine vehicles at the same stage and time slot are not optimal. The rest of proof is similar to Proposition \ref{optimality}.  
\end{proof}

In the above proof, we apply unit preemption in which each vehicle can only be preempted in unitary elements can be always replaced for ordinary preemption since the units can be made arbitrarily small as desired to approximate the normal preemption. 


To better explain the idea of how the heuristic is computed, we use the state presented in Figure \ref{state} as an example. The current state ends at stage $3$, and $7$ clusters are not scheduled yet. Checking the end times shown in Figure \ref{timeline}, all $7$ clusters will be shifted to the current end time of stage $3$ that is $t_6$ except the clusters on the connection $2$ and $6$. An unavoidable delay is incurred by all shifted vehicles ($\delta t_1$). For example, cluster $(2,3)$ is shifted to the end of cluster $(2,2)$, and cluster $(3,2)$ is shifted to $t_6$ with an added intergreen time. Based on the the proposed PDWSPT, we have to sort the clusters in terms of their shifted arrival time. The orderings will be thus $(2,3), (2,4), (6,3)$, and $(6,4)$. The remaining clusters, which are $(1,1)$, $(3,2)$ and $(7,3)$, will have the same $t_6$ plus the corresponding intergreen time as their current arrival time. The four clusters at stage $3$ are scheduled first without preemptions or any incurred delay. If we assume each vehicle's duration and weight are equivalent for simplicity, $(3,2)$ and $(7,3)$ cannot preempt the previous clusters. Their starting time will be postponed to the end of scheduled $(6,4)$, and an additional delay is incurred ($\delta t_2$). According to the stage transition of the signal cycle in Figure \ref{stage}, $(1,1)$ has to be scheduled after stage $4$ containing $(3,2)$ and $(7,3)$ and an additional delay is incurred ($\delta t_3$). The computed heuristic will be the summation $\delta t_1 + \delta t_2 + \delta t_3$.  

\begin{figure}[!htbp]
\centering
\includegraphics[scale = 0.2]{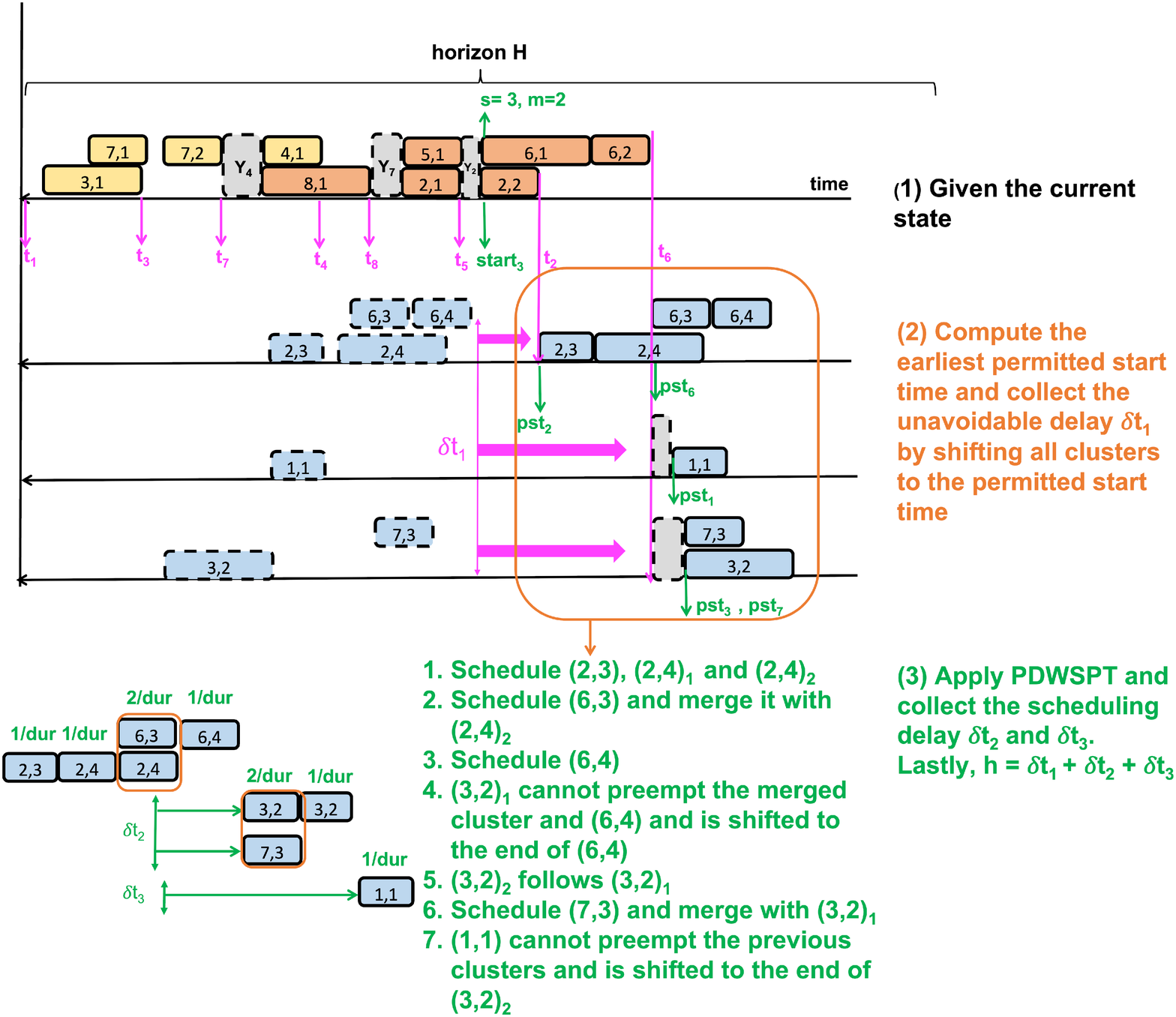}
\caption{Compute the heuristic given a partial schedule.}
\label{hfunction}
\end{figure}

\begin{algorithm}[tb]
\begin{algorithmic}[1]
\For{$c \in C_m(N_m), m \leq M$}
\State$S = updateState(S_p, c)$
\If{$sd$ of $S > max_s$}  insert $S$ to $maxSet$
\ElsIf{$sd$ of $S<min_s$} insert $S$ to $minSet$
\Else \quad insert $S$ to $candSet$
\EndIf
\EndFor
\If{$minSet$ is not empty} $candSet = minSet$
\ElsIf{$candSet$ is empty} $candSet = maxSet$
\EndIf
\For{$S \in candSet$}
\If{$domEqCheck(S)$}
\State insert $S$ to $openQueue$
\EndIf
\EndFor
\end{algorithmic}
\caption{$expandNeighbors(S_p)$}
\label{expand}
\end{algorithm}
\subsection{Pruning for Traffic Control}

In the following sections, four pruning checks based on the traffic control are performed before the expansions of the child state to reduce the number of search states are specified in Algorithm \ref{expand}. First, most traffic control systems implement maximum and minimum green time to ensure fairness. With these constraints, the candidate states that violate these constraints are not expanded. Second, a dominance check based on individual vehicle updates is performed to determine if the child states are dominated by the hashed states. Third, since we allow multiple non-conflicting connections to be served concurrently by the same stage, the states with different orderings of the connections on the same stage should be treated equivalently. 

\subsubsection{Maximum and Minimum Check}
As mentioned before, each stage has its own $max_s$ and $min_s$ to ensure fairness and allow drivers to react the start of the green interval. It is thus unnecessary to insert the search states, whose $\mathbf{t}$ is larger than $max_s$ or connection does not belong to current stage's connections if $sd$ < $min_s$, into the queue for expansions to reduce the search space. In Algorithm \ref{expand}, three sets are maintained for the expansion. First, we insert the child states that are over $max_s$ into $maxSet$. When we run into a child state whose end time is less than $min_s$, then it will be put into $minSet$. For those states falling between $max_s$ and $min_s$, they will be put in $candSet$. Now we determine which set to expand (line $8$ and $9$). If the $minSet$ is non-empty, it will be always expanded since any expansions not including them should not be optimal. Considering another case that the $candSet$ is empty, then the $maxSet$ is expanded without loss of optimality. otherwise, we expand the $candSet$.

\subsubsection{Dominance and Equivalence Check}
The search algorithm performs dominance checks on the states in the $candSet$ (line $12$). We hash the states based on $\mathbf{N}$ and $m$ that are converted to an unique identifier. The new child state $s$ is dominant over a hashed state $s'$ if its cumulative delay $d$ is less than the hashed state's $d$ and the $t_m(s) \leq  t_m(s')$ for all connections. In other word, $s$ is serving the same number of vehicles on each lane with lower cost and more timing flexibility to accommodate more vehicles. 

Other than dominance check, we also check if the current path is equivalent to the previous expanded paths. Since multiple connections are allowed to be served simultaneously at the same stage, it would be redundant to traverse those paths. Whenever a new stage $s$ is brought up and the set $\mathcal{M}_s$ does not overlap the previous one, the distribution of the number of vehicles on the connections is tracked and compared. If they are equal, then this new state is unnecessary to be expanded. For instance, serving connections $1,2,1$ is equivalent to serving $2,1,1$. This check can be ignored if we remove the cycle constraint that all stage should be following a certain ordering.

A* search usually maintains a closed list to ensure that we will not expand the same states more than once. As described before, the states are defined by both discrete and continuous features including the stage, connection and the end times so on. Thus, the search is a \textit{tree search}, and visiting the exact same state twice is less likely. For this reason, it is unnecessary to do a closed list check as is typical in A* search, as previously observed in \cite{goldstein2019expressive}.

\section{Experimental Evaluation}
In this section, we utilize three ways to verify the efficacy of the proposed heuristic scheduling system. In timing analysis, our goal is to verify that the proposed PDWSPT heuristic returns a schedule with less expansions and time (i.e., it is a tighter lower bound) than the heuristic previously proposed for lane-based intersection scheduling. In the simulation evaluation, we demonstrate that the new detailed model with the proposed heuristic outperforms the previous heuristic and a less expressive, approach-based intersection scheduling method implemented using Dynamic Programming (DP) \cite{xie2012schedule}. In the field tests, an on-off evaluation is conducted on a corridor network over two weeks to measure if the proposed approach outperforms the DP approach. 

\subsection{Timing Analysis}
For the timing analysis, we compare the proposed PDWSPT heuristic with the one proposed in ERIS \cite{goldstein2019expressive}. The main difference between these two heuristics is that the ERIS heuristic divides the problem into multiple preemptive subproblems, while PDWSPT adapts an exact solution to the same scheduling problem with preemption. To better understand the effect of these two heuristics, the cycle constraints are lifted to allow the algorithm to determine the next stage (i.e., the equivalence check is thus unnecessary), so that the heuristic will be a tighter lower bound to the problem. The search aims to determine the optimal stage sequences for minimizing the cumulative delay with a compute time limit of $5$ seconds. In addition, we also experiment with cases that include dominance checks only, dominance and maximum/minimum checks, no checks, and no heuristics. It is worthwhile to note that the A* search can be transformed to the Dijkstra algorithm by turning off the heuristics (i.e., $f = g$ instead of $f = g + h$). 

We prepare $1200$ test instances whose number of connections is from $3$ to $6$ and the number of vehicles ranges from $4$ to $46$ vehicles, and the data are generated continuously by a running simulation. For a tree search, the worst case of the state updates is $O(M^{\sum_{i=1}^MN_i})$. The problem instances are thus sufficient to provide various complexities and traffic patterns. Two metrics, the run-time performance and number of search state expansion percentiles, are collected and presented in Table \ref{profiling}.

Running the heuristic search with all pruning checks is seen to achieve the lowest number of state updates, comparing to other variations. Applying maximum and minimum checks only provides a small reduction in the state updates since we adopt a small minimum green time here. For the first $6$ rows of Table \ref{profiling}, the run-times are all below $500$ms for $75\%$ of time and can be reduced further if the cycle constraints are applied. If we look at both metrics, it implies that the run-time performance not only depends on the state expansions but also on the computational cost of running checks (iterate through all hashed states). The heuristics with all checks outperform the cases using only the heuristic $75\%$ of time in terms of both metrics, although the mean run-time performances are only comparable.

If we specifically compare PDWSPT with the ERIS heuristics, PDWSPT's run-time performance and the number of the state updates are always better than ERIS. For example, the number of PDWSPT's expansions with all checks is less than the same case of ERIS by $8\%$ at the $75\%$ percentile. PDWSPT provides more efficient and tighter lower bound as a heuristic of A*. In addition, the case without using heuristics (i.e., Dijkstra) shows that the heuristic function is quite essential for the scheduling algorithm. Without using the heuristic function, the search requires an order of magnitude increase in run-time to find the solution.

\begin{table}[!t]
\centering

		\centering
		 \scalebox{0.55}{
		 \begin{tabular}{*{11}{c}}
   \toprule
      \multirow{2}{*}{}& \multicolumn{5}{c}{Run-Time (ms)}  & \multicolumn{5}{c}{ Number of Expansions }  \\
       \cmidrule(l){2-6}\cmidrule(l){7-11}    
     \multirow{10}{*}{}& \multicolumn{1}{c}{ Mean}  &\multicolumn{1}{c}{$25\%$}  & \multicolumn{1}{c}{ $50\%$ }&\multicolumn{1}{c}{$75\%$} &\multicolumn{1}{c}{$95\%$}  & \multicolumn{1}{c}{ Mean}  &\multicolumn{1}{c}{$25\%$}  & \multicolumn{1}{c}{ $50\%$ }&\multicolumn{1}{c}{$75\%$} &\multicolumn{1}{c}{$95\%$} \\
  
     \midrule
   PDWSPT+D,M&603 & 14 & 59 & 306 & 5001 & 19288 & 1259 & 4681 & 18644 & 99722\\
  ERIS+D,M&617 & 15 & 62 & 319 & 5001 & 20781 & 1339 & 4978 & 20233 & 110738\\
  \midrule
  PDWSPT+D&609 & 15 & 60 & 322 & 5001 & 19976 & 1420 & 4905 & 19678 & 103095\\
  ERIS+D&645 & 16 & 67 & 352 & 5001 & 20770 & 1563 & 5361 & 21371 & 105184 \\
 \midrule
 PDWSPT&574 & 19 & 79 & 435 & 4230 & 53596 & 2102 & 8136 & 43133 & 294840\\
  ERIS&613 & 21 & 90 & 461 & 4536 & 56764 & 2371 & 9028 & 45795 & 330119\\
  \midrule
  Dijkstra+D,M& 2509 & 261 & 1734 & 5001 & 5001 & 91120 & 21258 & 85308 & 144850 & 220114\\
  Dijkstra&2969 & 600 & 3767 & 5001 & 5001 & 300546 & 77113 & 285931 & 484821 & 678883\\

      \bottomrule

  \end{tabular}
  }
  \caption{Run-time and number of state expansions comparisons for single intersection. (D: dominance check; M: maximum/minimum check)}
\label{profiling}
\end{table}

 \begin{table}[!t]
\centering		
\centering
\scalebox{0.6}{
\begin{tabular}{*{5}{c}}
   \toprule
      \multirow{2}{*}{}& \multicolumn{2}{c}{High demand}  & \multicolumn{2}{c}{Low demand}  \\
       \cmidrule(l){2-3}\cmidrule(l){4-5}    
     \multirow{4}{*}{}& \multicolumn{1}{c}{ Delay(s)}  &\multicolumn{1}{c}{No. of Stops}  & \multicolumn{1}{c}{ Delay(s)}&\multicolumn{1}{c}{No. of Stops} \\
     \midrule
 \textit{Marshall(6 ints.)}&&&&\\
     \midrule
      
 PDWSPT&$55.15\pm41.07$&$2.09\pm1.72$&$39.41\pm33.58$&$1.43\pm1.13$\\
ERIS&$57.08\pm42.92$&$2.21\pm1.90$&$40.19\pm34.25$&$1.43\pm1.14$\\
Dijkstra&$58.71\pm47.53$&$2.19\pm2.00$&$42.24\pm35.29$&$1.43\pm1.16$\\
DP&$65.22\pm56.65$&$2.44\pm2.60$&$40.89\pm34.29$&$1.42\pm1.11$\\
Actuation&$77.68\pm63.86$&$2.64\pm2.32$&$37.83\pm29.70$&$1.59\pm1.23$\\
    \midrule
     \textit{St. Albert(11 ints.)}&&&&\\
      \midrule
   PDWSPT&$50.82\pm47.20$ & $1.58\pm1.45$&$43.81\pm44.71$&$1.32\pm1.17$\\
  ERIS&$51.78\pm48.02$ & $1.58\pm1.43$&$46.12\pm46.49$&$1.35\pm1.19$\\
  Dijkstra&$52.37\pm48.86$ & $1.57\pm1.45$&$44.75\pm45.58$&$1.32\pm1.18$\\
  DP&$54.9\pm51.98$ & $1.58\pm1.49$&$45.43\pm45.68$&$1.32\pm1.15$\\
  Actuation&$52.37\pm51.69$ & $1.92\pm1.97$&$34.98\pm33.55$&$1.57\pm1.49$\\

      \bottomrule

  \end{tabular}
  }
  \caption{The delay and number of stops comparison with two networks}
\label{marshall}
\end{table}

\subsection{Simulation Results}
In this section, we compare the above PDWSPT heuristic to four other real-time traffic control methods. First, we compare the proposed heuristic with the ERIS heuristic and Dijkstra with the same lane-based model. Second, the less-expressive DP approach of \cite{xie2012schedule} is compared. \footnote{Note also that previous research with the baseline schedule-driven approach has shown its comparative advantage over prior real-time traffic signal control approaches \cite{xie2012schedule}.} Third, a fully-actuated control method wherein all stages are extended by actuation is implemented and compared. The schedule generation time limit in these experiments was set at $2$ seconds, with all methods returning the best suboptimal schedule within the time limit.

To evaluate the proposed heuristic, we simulate delay and the number of stops over two real world networks under two realistic traffic patterns of different demands. The first is a smaller highway interchange network that has $6$ intersections and complex traffic flows from all directions ($3710/$hour and $1855/$hour). The second is a lengthy corridor network that has $11$ intersections and higher demand over the main road ($5484/$hour and $2742/$hour). These two networks, as shown in Figure \ref{simnets}, are representative for understanding how the heuristics perform under different scenarios. The simulation model was in VISSIM, a commercial microscopic traffic simulation software package. Tested traffic volume is averaged over sources at network boundaries. To assess the performance boost provided by the proposed algorithm, we measure the average waiting time of all vehicles over five runs. All simulations run for $1$ hour of simulated time. 

For the first network, PDWSPT's delay and number of stops are always better than the other three online planning approaches. Our proposed heuristic is able to quickly find the optimal solution and achieve a better real-time responsiveness to the dynamic traffic flow. When the demand is higher, the advantage becomes more obvious. The results of the corridor model shows the same trend. For the actuation approach, although it can achieve a lower delay under the low demand, it increases the number of stops, another important metric in evaluating traffic control.

\subsection{Field Test Results}
To demonstrate the potential of heuristic scheduling methods in the field, a heuristic scheduling real-time system with PDWSPT was deployed on a corridor at Chicopee, Massachusetts. The corridor consists of $13$ intersections shown in Figure \ref{fieldmap}. To evaluate the delay improvements, we used a third-party detector system to evaluate the delay of south bound (main flow) during PM rush hour (i.e. 3PM -6PM). The delay was obtained by averaging the time between stop-bar detector actuation and when the light turns green.We collected traffic data for two weeks, the first week with the DP algorithm running, and the second week with A* with the PDWSPT heuristic running. Table \ref{chicopee} 
summarizes the delay improvement in seconds. 
The results show that $11$ out of $13$ intersections achieved lower average delay with PDWSPT heuristic than the DP method, and thus the average travel time of PM rush hour was also reduced by $3\%$. 
\begin{table}[!htbp]
\centering
\scalebox{0.6}{
\begin{tabular}{  l|lcc|l|ccc } 
 Intersection &DP(s)&A*(s)& Diff($\%$) & Intersection &DP(s)&A*(s)& Diff($\%$)    \\  \hline
Aldi &49.0&38.1&28$\%$&Abbey&16.8&10.2&64$\%$\\ 
Mass Pike. &79.9&82.9&-3$\%$&Granby&12.3&9.1&35$\%$\\ 
Pendleton &55.9&50.3&11$\%$&Irene &25.2&25.9&-2$\%$\\  
Britton&27&21&27$\%$&New Ludlow&50.6&40.4&25$\%$\\
James W.&37.8&38.6&-2$$&James E.&4.2&4.7&-10$\%$\\ 
Fairview&25.5&19.2&$32\%$&Streiber&25.1&18.2&37$\%$\\ 
Marketplace&48.6&46.8&3$\%$&Fuller&16.6&15.1&9$\%$\\\hline
\textbf{Travel Time}&$\mathbf{8m24}$&$\mathbf{8m7}$&$\mathbf{3\%}$&&&&
\end{tabular}
}
\caption{$11$ out of $13$ intersections achieve smaller delay with PDWSPT at Chicopee, MA during PM rush hours (3pm-6pm) over 5 days (4/21-4/25 v.s. 7/11-7/15, 2022).}
\label{chicopee}
\end{table}

\begin{figure}
\begin{minipage}{.2\textwidth}
  \centering
  \includegraphics[scale=0.22]{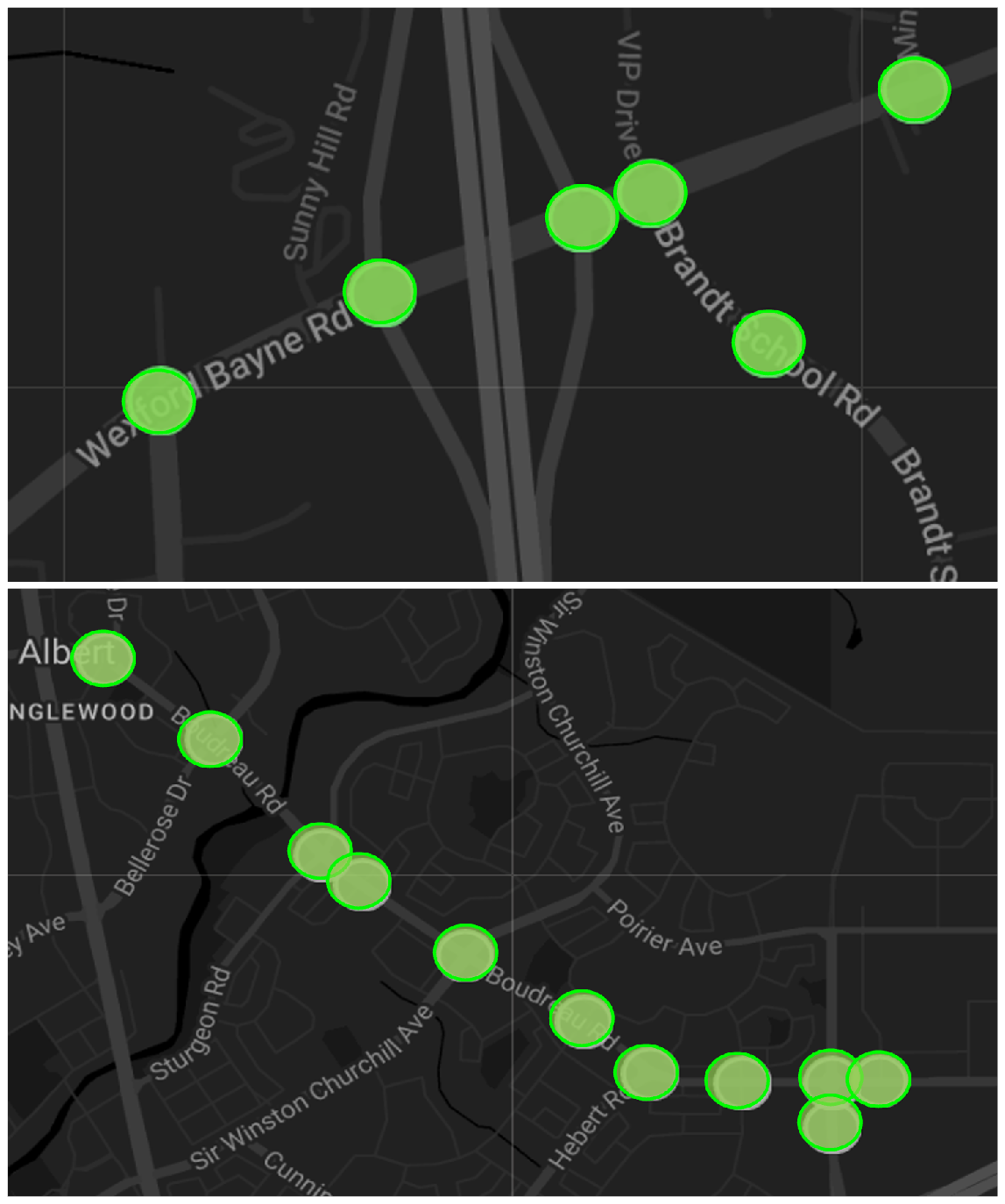}
  \captionof{figure}{Two simulation networks}
  \label{simnets}
\end{minipage}
\begin{minipage}{.25\textwidth}
  \centering
  \includegraphics[scale=0.35]{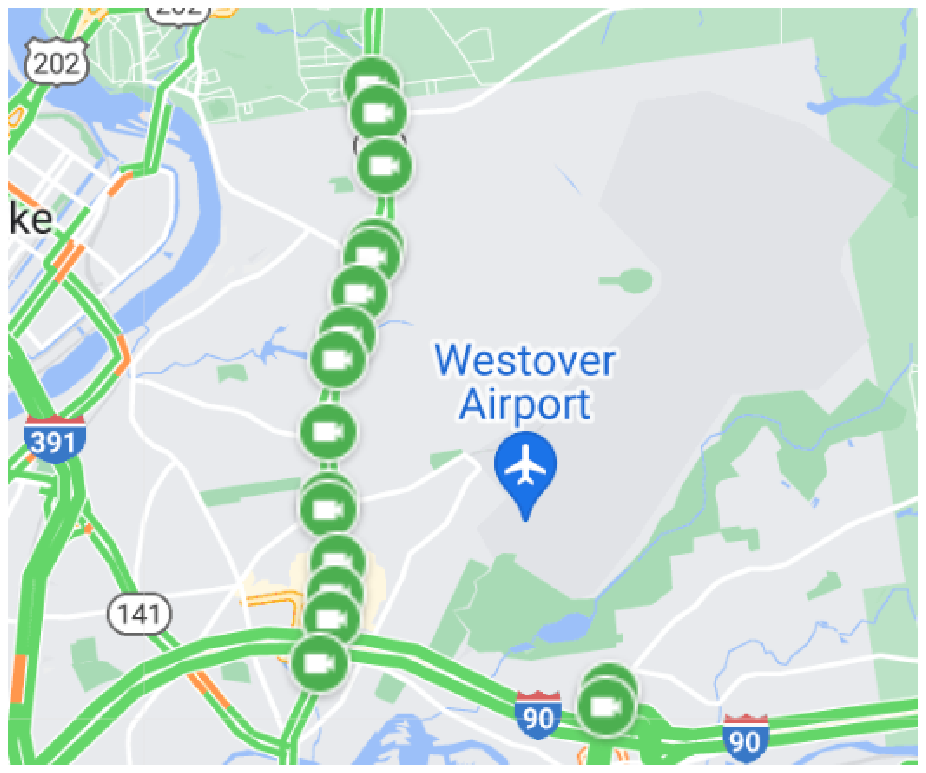}
  \captionof{figure}{13 intersections on Memorial Dr. of Chicopee, MA.}
  \label{fieldmap}
\end{minipage}
\end{figure}

\section{Conclusions}
In this paper, we have described a new efficient heuristic for lane-based optimization of traffic flows through signalized intersections. The lane-based formulation that is considered is more expressive than those considered previously by schedule-driven traffic control approaches, resulting in more accurate prediction of when sensed vehicles will arrive at the intersection and a stronger basis for optimization.
The proposed heuristic function is key to managing the combinatorics of the resulting state space,
 optimally solving a unit-preemptive relaxation of the original lane-based scheduling problem to provide a tight lower bound for A* search. Further efficiency is gained by incorporating four pruning checks to limit state expansion.
Experimental results and field evaluations showed that our heuristic scheduling approach outperforms previous schedule-driven approaches along both efficiency and effectiveness performance dimensions. 



\bibliographystyle{aaai23}
\bibliography{a_star}  

\end{document}